\documentclass{article}

    \PassOptionsToPackage{numbers, compress}{natbib}
 \usepackage[preprint]{neurips_2026}

\usepackage[utf8]{inputenc} 
\usepackage[T1]{fontenc}    
\usepackage{hyperref}       
\usepackage{url}            
\usepackage{booktabs}       
\usepackage{amsfonts}       
\usepackage{nicefrac}       
\usepackage{microtype}      
\usepackage{xcolor}         

\usepackage{amsmath}
\usepackage{amsthm}
\usepackage{algorithm}
\usepackage{algpseudocode}
\usepackage{graphicx}
\usepackage{multirow}
\usepackage{wrapfig}
\usepackage{caption}

\newtheorem{definition}{Definition}
\newtheorem{theorem}{Theorem}

\newtheorem{proposition}{Proposition}

\newtheorem{remark}{Remark}

\title{AegisFlow: Training-Free Non-Myopic Path-Safe Guided Flow Matching}

\author{%
  Kunpeng Liu\\
  Department of Automation and BNRist\\
  Tsinghua University\\
  Beijing, China \\
  \texttt{liukp22@mails.tsinghua.edu.cn}\\
  \And
  Boshi Zhang\\
  Department of Automation and BNRist\\
  Tsinghua University\\
  Beijing, China \\
  \texttt{zbs22@mails.tsinghua.edu.cn}\\
  \And
  Keyou You\thanks{Corresponding author}\\
  Department of Automation and BNRist\\
  Tsinghua University\\
  Beijing, China \\
  \texttt{youky@tsinghua.edu.cn}\\
}

\begin{document}

\maketitle

\begin{abstract}
  Pretrained generative models, such as diffusion and flow matching models, have demonstrated exceptional capabilities in synthesizing high-quality data. However, steering their generative trajectories to satisfy strict constraints—such as identity preservation in image editing or obstacle avoidance in robotics—remains fundamentally challenging. While recent training-free guidance methods enforce path-safety, they rely heavily on min-norm controllers that only account for immediate constraint satisfaction. These myopic approaches often yield suboptimal trajectories that hover dangerously close to the boundary, severely restricting the model's ability to optimize toward its target. In this paper, we propose AegisFlow, a training-free, non-myopic path-safe guidance framework that recasts guided generation as a safety-critical optimal control problem with global scope. By integrating a time-varying barrier function directly into the global cost functional, AegisFlow naturally enforces constraints along the entire generative path. To compute the global optimal control, we derive a tractable, mathematically grounded approximation of the Hamilton-Jacobi-Bellman (HJB) equation. Comprehensive experiments demonstrate that AegisFlow outperforms state-of-the-art methods across text-guided image manipulation and robotic planning, establishing a significantly better trade-off between target alignment and constraint satisfaction. Code is available at \url{https://github.com/LiuKunpeng03/AegisFlow}.
\end{abstract}

\section{Introduction}

\begin{figure}
  \centering
  \includegraphics[width=0.85\linewidth]{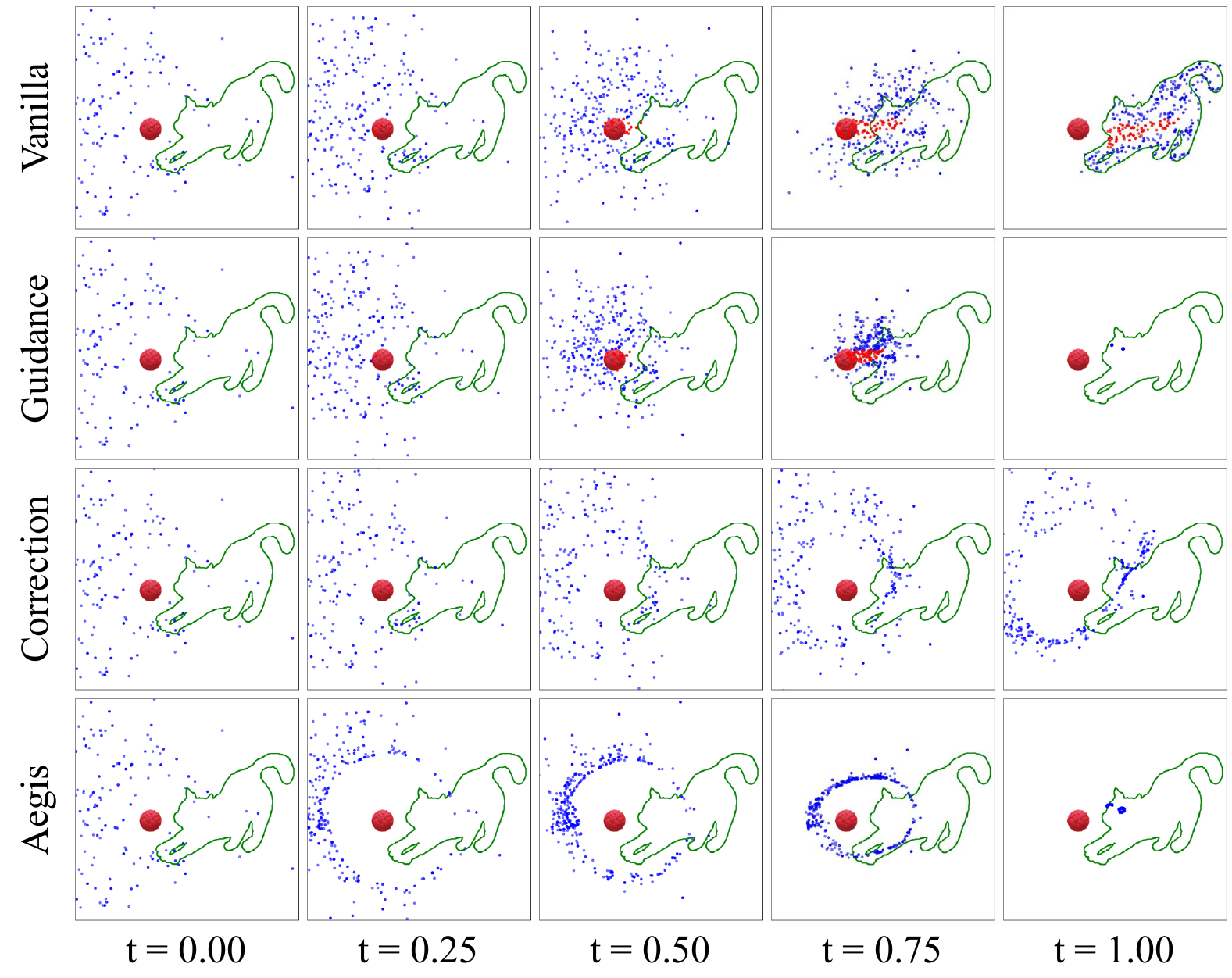}
  \caption{\textbf{Toy Example.} \textbf{Vanilla} represents the original flow which is trained to fit into the contour of a cat. \textbf{Guidance} stands for the case where the flow is guided towards two eyes of the cat. Although the guidance successfully steers the trajectory, many trajectories collide with the red yarn ball which serves as the obstacle. \textbf{Correction} means min-norm controller is applied to the former case. While all trajectories are safe, the quality is significantly degraded. \textbf{Aegis} successfully guides most of the trajectories towards the target while maintaining a safe distance from the obstacle. }\label{fig:toy_example}
\end{figure}

Generative models such as diffusion~\citep{ho2020denoising,song2020score} and flow matching~\citep{lipman2022flow} have exhibited remarkable capabilities in synthesizing high-quality data across various domains, including images~\citep{rombach2022high,esser2024scaling} and planning~\citep{carvalho2023motion,ding2025fast}. Guided generation, e.g., image editing~\citep{meng2021sdedit,hertz2023delta}, takes advantage of the expressiveness of pretrained generative models and steers the generative trajectory towards a target. Reward-guided methods~\citep{song2023loss,liu2023flowgrad,wang2024training}, which leverage a predefined reward function to optimize the trajectory at inference time, have attracted significant attention due to their \emph{training-free} nature and flexibility in handling various objectives. 

Besides objective guidance, constraint satisfaction is increasingly accounted for in guided generation. For example, fake or harmful content should be strictly prohibited from generation~\citep{schramowski2023safe,li2025responsible}, and generated trajectories for robotics should obey physical law and avoid obstacles~\citep{janner2022planning,carvalho2023motion}. Previous methods deal with constraints by either soft guidance~\citep{yuan2023reward,chen2024simple} or projection~\citep{christopher2024constrained,utkarsh2025physics,zampini2025training}. The former admits no hard guarantees on constraint satisfaction, while the latter might hinder the generation quality. Recent works~\citep{xiao2023safediffuser,yang2025safeflowmatcher,gadginmath2026provably} extend the desired constraint on the output to the entire generative trajectory, and enforce \emph{path-safety} by imposing control to the pretrained flow model. These methods are able to achieve better constraint satisfaction while maintaining generation quality. However, they apply a \emph{myopic} min-norm controller, which only considers the immediate effect of the control input without accounting for the long-term effect on the trajectory. As a result, these methods often find suboptimal solutions that hover dangerously close to safety boundaries, severely limiting the aggressiveness with which a model can optimize toward its target. See Figure~\ref{fig:toy_example} for an illustration.

To overcome the fundamental limitations of myopic controllers, we propose AegisFlow, a \emph{training-free, non-myopic path-safe} guidance method for flow matching models. AegisFlow recasts guided generation as a safety-critical optimal control problem with a global scope. Instead of relying on per-step corrections, we integrate a time-varying barrier function directly into the global cost functional. This formulation penalizes the trajectory for approaching the boundary of the safe set, ensuring the entire generated path remains within a safe tube. AegisFlow seeks to find the global optimal control via estimating the solution to the Hamilton-Jacobi-Bellman (HJB) equation associated with the cost functional, which inherently accounts for the long-term effects of control inputs on both target alignment and safety constraints. Since the barrier function drives the cost to infinity near safety boundaries, any finite-cost solution natively guarantees constraint satisfaction, allowing us to bypass the ad-hoc projections or per-step barrier functions required by prior methods.

We extensively evaluate AegisFlow across diverse and challenging domains, spanning from high-dimensional text-guided image manipulation (CelebA-HQ) to physical planning (Maze2D and D3IL robotic manipulation). In image editing, AegisFlow proactively maintains a larger safety margin for identity preservation, enabling stronger textual alignment (higher CLIP scores) and better identity retention than myopic baselines. In continuous control tasks, AegisFlow achieves higher success rates, smoother trajectories, and fewer safety violations compared to current state-of-the-art safe diffusion methods.

In summary, our main contributions are as follows:

\begin{itemize}
\item \textbf{Overcoming of Myopic Limitations:} We identify the suboptimality of min-norm controllers in guided flow matching and propose AegisFlow, a globally aware, training-free optimal control framework.
\item \textbf{Tractable HJB Approximation:} We formulate a barrier-inclusive cost functional and derive a tractable, mathematically grounded approximation for the optimal safe control using the Hamilton-Jacobi-Bellman (HJB) equation.
\item \textbf{Comprehensive Empirical Validations:} We demonstrate superior performance across both high-dimensional image manipulation and safety-critical continuous control tasks, establishing a significantly better trade-off between target alignment and constraint satisfaction.
\end{itemize}

\section{Related Works}

\paragraph{Training-free image manipulation.} There are various training-free methods for image manipulation based on pretrained diffusion models. The mainstream approach starts with image-to-noise inversion and then applies certain guidance to the generative trajectory~\citep{liu2023flowgrad,samuel2023lightning,brack2024ledits++,huberman2024edit,ben2024d,wang2024training}. Some works directly optimize over the original image~\citep{koo2024posterior,nam2024contrastive,yang2024text}. Recently, \citet{kulikov2025flowedit} also explore inversion-free methods. In this paper, we will mainly focus on the first approach. However, we think there is no fundamental barrier to applying our method to the other approaches, and we leave it for future work.

\paragraph{Flow matching under control theory.} Prior works have formulated the training~\citep{domingo2024adjoint,zhao2025score,zhu2025unidb}, sampling~\citep{zhang2021path,berner2022optimal,vargas2023denoising,havens2025adjoint}, and test-time editing~\citep{liu2023flowgrad,wang2024training,rout2024semantic,rout2025rb,gao2026terminally} of flow matching models as optimal control problems. Some of them use Pontryagin's Maximum Principle (PMP)~\citep{domingo2024adjoint,liu2023flowgrad,wang2024training}, while others adopt the Hamilton-Jacobi-Bellman (HJB) approach~\citep{zhao2025score,rout2025rb,gao2026terminally}. They aim to optimize a certain cost, while the behavior of the trajectory is controlled by a regularizer and admits no hard guarantees. Several recent works explore the use of per-step projection~\citep{christopher2024constrained,romer2024diffusion,zampini2025training,utkarsh2025physics,luan2025projected}, control barrier functions (CBFs) with min-norm controller (MNC)~\citep{xiao2023safediffuser,zhang2025constrained,yang2025safeflowmatcher,gadginmath2026provably} to enforce hard constraints along the entire trajectory. However, the myopic nature of per-step projection and MNCs remains to be addressed. \citet{mizuta2024cobl} try to encode the CBF constraint into the reward function. But they do so by including the increment of the term inside the supremum in \eqref{eq:cbf_condition} (in Appendix~\ref{app:cbf}) between two timesteps as part of the reward, which is still essentially myopic and cannot find globally optimal solutions. \citet{mishra2025eb} explore the use of an emerging barrier function to enforce safety constraints in the diffusion process, but their method is mainly a heuristic one. In constrast, our method estimates the optimal control from the HJB equation with a cost functional that includes the barrier function, which attempts to find a globally optimal solution that minimizes the cost functional while ensuring the trajectory stays within a safe set.

\section{Preliminaries and Problem Formulation}

\paragraph{Flow Matching.} Flow matching~\citep{lipman2022flow, liu2022flow,albergo2025stochastic} formulates the diffusion models from a perspective of ordinary differential equations (ODEs). Given a data distribution $\pi_1$ and a noise distribution $\pi_0$, flow matching learns a time-dependent vector field $f(x, t)$ that transports $\pi_0$ to $\pi_1$ along the trajectories defined by the ODE:
\begin{equation}\label{eq:flow_matching_ode}
\dot{x} = f(x, t),\quad x(t_0)\sim\pi_0.
\end{equation}

Indeed, once a sample $x_0$ is drawn from $\pi_0$, a unique trajectory $x_t$ is determined by the ODE under necessary regularity conditions. Therefore $f$ induces a map $\Phi^f_{t_1\to t_2}$ such that $\Phi^f_{t_1\to t_2}(x(t_1)) = x(t_2)$, which further results in a pushforward measure $\Phi^f_{t_0\to t_f}(\pi_0)$. The flow matching objective is to learn a vector field $f$ such that $\Phi^f_{t_0\to t_f}(\pi_0) = \pi_1$. In practice, flow matching generally minimizes the regression objective:
\[\mathbb{E}_{t,x(t_0),x(t_f)}(\|f(x(t),t)-u_t(x(t) | x(t_0),x(t_f))\|_2^2),\]
where $u_t$ is the velocity field of the reference trajectory, which can be defined in various ways, e.g. linear interpolant~\citep{albergo2025stochastic}.

\begin{figure}
  \centering
  \includegraphics[width=1.0\linewidth]{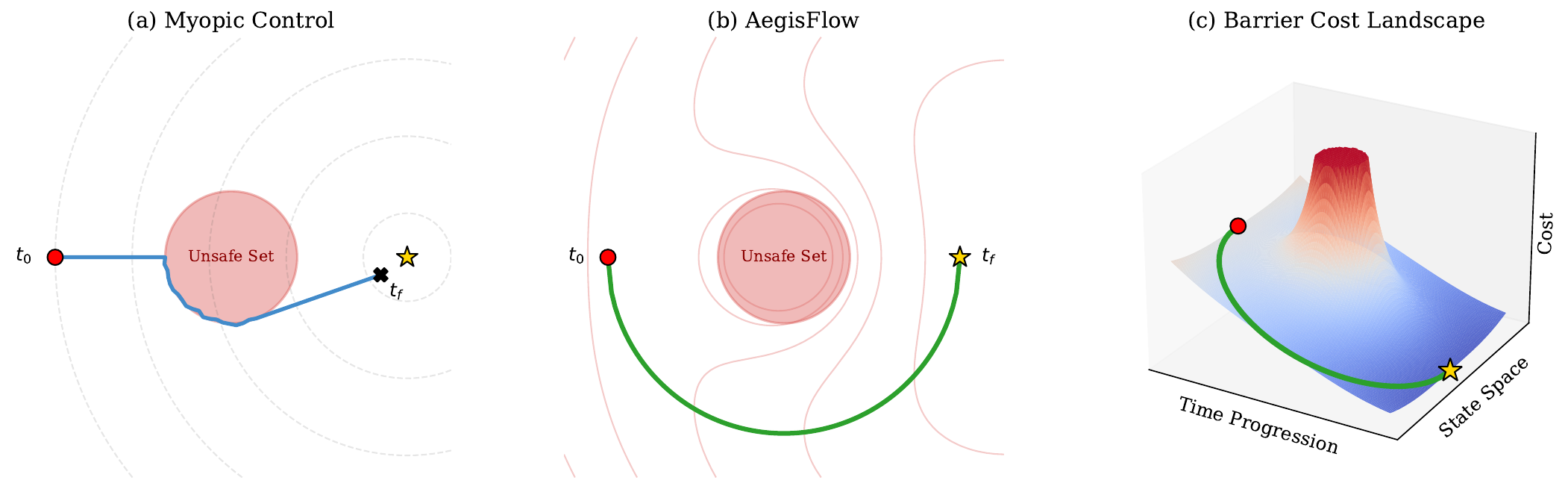}
  \caption{\textbf{Conceptual comparison between myopic control and AegisFlow.} \textbf{(a) Myopic Control}: A greedy controller driven purely by local target attraction and rigid boundary constraints. It gets trapped sliding jaggedly along the edge of the Unsafe Set and fails to reach the target state. \textbf{(b) AegisFlow}: By explicitly modeling the global cost field, the proposed trajectory is smooth and globally optimal. It proactively avoids the obstacle with an anticipated safety margin and successfully converges to the target. \textbf{(c) Barrier Cost Landscape}: The cost functional used by AegisFlow combines a barrier function to enforce constraints and a reward of the target, balancing constraint satisfaction and target guidance.}\label{fig:concept}
\end{figure}

\paragraph{Guided Generation and Optimal Control.} 

The primary goal of guided generation is to steer the generative trajectory towards a target while satisfying some constraints, e.g. identity preservation or obstacle avoidance. In the scope of control theory, the problem can be formulated as follows. Given a pretrained flow matching model with vector field $f$, we consider the following controlled dynamics:
\begin{equation}\label{eq:controlled_dynamics}
\dot{x} = f(x,t) + u,
\end{equation}
where $u$ is the control input to be chosen. Reward-guided methods prove to be effective in steering the flow towards the target. Specifically, they select the control input $u$ that minimizes a certain cost functional,
\begin{equation}\label{eq:cost_functional}
J(u):=\int_{t_0}^{t_f}L(x(t), u(t), t)\mathrm{d}t + \Psi(x(t_f)),\quad \text{s.t.} \quad x(t_0)=x_0,
\end{equation}
where $L$ is the running cost and $\Psi$ is the terminal cost. For example, \citet{wang2024training} set $L(x, u, t)=\|u\|_2^2$ and $\Psi(\cdot)=-\operatorname{CLIP}(\cdot, \text{prompt})$ in \eqref{eq:cost_functional} in the text-guided image manipulation task to encourage prompt adherence.

Finding the optimal control $u^\star$ that minimizes the cost functional \eqref{eq:cost_functional} is a classical problem in control theory, which can be solved by the Hamilton-Jacobi-Bellman (HJB) equation. To begin with, let us make the dependence of the cost functional $J$ on the initial state $x_0$ and time $t_0$ explicit, namely
\begin{equation}\label{eq:cost_functional_flow_matching}
J(u, x_0, t_0):=\int_{t_0}^{t_f}L(x(t), u(t), t)\mathrm{d}t + \Psi(x(t_f)),\quad \text{s.t.} \quad x(t_0)=x_0.
\end{equation}
Denote the value function of the optimal control problem as
\begin{equation}\label{eq:value_function}
V(x, t):=\min_{u} J(u, x, t).
\end{equation}
HJB states that under necessary regularity conditions, $V$ satisfies the following PDE:
\begin{equation}\label{eq:hjb}
-\partial_t V(x, t) = \min_{u} [L(x, u, t) + \nabla_x V(x, t)\cdot (f(x, t)+u)],\quad \text{s.t.} \quad V(x, t_f)=\Psi(x),
\end{equation}
and the optimal control $u^\star$ satisfies
\begin{equation}\label{eq:optimal_control_hjb}
u^\star(x, t) \in \arg\min_{u} [L(x, u, t) + \nabla_x V(x, t)\cdot (f(x, t)+u)].
\end{equation}
Conversely, if a function $\hat{V}$ satisfies \eqref{eq:hjb} and a control $\hat{u}$ satisfies \eqref{eq:optimal_control_hjb} with $V$ replaced by $\hat{V}$, then it must be the case that $\hat{V}$ is the value function and $\hat{u}$ is an optimal control.

Constraint satisfaction is also increasingly considered in guided generation. It requires the generated sample to satisfy certain constraints, i.e., $x(t_f)\in \mathcal{C}$ for some constraint set $\mathcal{C}$. Traditional approaches utilize soft guidance or projection to enforce constraints, which offer no hard guarantees or suffer from performance degradation. For example, \citet{gao2026terminally} set $\Psi$ to be the distance to the admissible set and utilize HJB to estimate the optimal control. To alleviate performance deterioration caused by post hoc projection, recent works extend the constraints to the entire trajectory~\citep{xiao2023safediffuser,yang2025safeflowmatcher,gadginmath2026provably}, i.e., they demand $x(t)\in \mathcal{C}_t$ for all $t\in[t_0,t_f]$ where $\mathcal{C}_t$ is a time-varying safe set with $\mathcal{C}_{t_f}\subseteq\mathcal{C}$. They enforce such \emph{path-safety} by CBFs and MNCs. However, MNCs only consider the immediate effect of the control input and such myopic nature often leads to suboptimal solutions. Details about CBFs and MNCs are deferred to Appendix~\ref{app:cbf}.

In summary, the core problem we seek to address is as follows: \emph{How can we find the optimal control that minimizes the cost functional while satisfying the given constraints?} Formally speaking, we aim to solve the following optimization problem:
\begin{equation}\label{eq:optimal_control_problem}
\begin{aligned}
\min_u& \qquad J(u, x_0, t_0):=\int_{t_0}^{t_f}L(x(t), u(t), t)\mathrm{d}t + \Psi(x(t_f)),\\
\text{s.t.}& \qquad \dot{x} = f(x,t) + u, \quad x(t_0)=x_0, \quad x(t)\in \mathcal{C}_t, \quad \forall t\in[t_0,t_f].
\end{aligned}
\end{equation}

\section{Methodology}

\begin{algorithm}[t]
\caption{AegisFlow}
\label{alg:aegisflow}
\small
\begin{algorithmic}[1]
\State \textbf{Input:} ${x}_{t_0}\sim\rho_0$, schedules $\{\lambda_t,\beta_t\}$, barrier function $B$, reference flow $f$, terminal loss $\Psi$, timestep grid $\{t_0,\dots,t_N\}$, maximum iteration $M$, learning rate $\eta$, weight decay $\nu$
\For{$k=0,\dots,M-1$}
  \State Solve the controlled dynamics \eqref{eq:controlled_dynamics} with the current control $u^k_n$ to obtain the trajectory $\{x_n=x(t_n)\}_{n=0,1,\dots,N-1}$
  \State $s_n\gets\int_{t_n}^1\lambda_\tau^{-1}\mathrm{d}\tau$
  \State $g_{\mathrm{safe},n}\gets \nabla_{{x}}B({x}_n,t_n)$
  \State ${g}_{\mathrm{tar},n}\gets \nabla_{{x}}(\Psi\circ\Phi^{f+u^k}_{t_n\to1})({x}_n)$
  \State $g_n\gets \lambda_{t_n}^{-1}{g}_{\mathrm{tar},n}+s_n\beta_{t_n}g_{\mathrm{safe},n}$
  \State $u^{k+1}_n\gets \nu u^k_n-\eta g_n$
\EndFor
\State \textbf{return} ${x}_{t_N}$
\end{algorithmic}
\end{algorithm}

\subsection{AegisFlow: Non-Myopic Path-Safe Guided Flow Matching}

We have recast guided generation as a safety-critical optimal control problem \eqref{eq:optimal_control_problem}. To enhance the long-term awareness of the safe control, we consider the following cost functional:
\begin{equation}\label{eq:cost_functional_aegis}
J(u, x_0, t_0):=\int_{t_0}^{t_f}\Big[\frac{\lambda_t}{2}\|u(t)\|_2^2+\beta_t B(x(t),t)\Big]\mathrm{d}t + \Psi(x(t_f)),\quad \text{s.t.} \quad x(t_0)=x_0,
\end{equation}
where $B(\cdot,t):\operatorname{Int}(\mathcal{C}_t)\to\mathbb{R}_{\ge0}$ is a barrier function that satisfies
\[\lim_{x\to\partial \mathcal{C}_t} B(x,t) = +\infty, \quad \forall t\in[t_0,t_f].\]

Note that the safety constraints are implicitly enforced by the barrier function in the cost functional, which allows us to directly apply the HJB approach to estimate the optimal control. Intuitively, the barrier function $B$ penalizes the trajectory for getting too close to the boundary of the safe set, which encourages the trajectory to stay within $\mathcal{C}_t$. See Figure~\ref{fig:concept} for an illustration. Indeed, if the cost functional \eqref{eq:cost_functional_aegis} is finite, then the trajectory must satisfy $x(t)\in\operatorname{Int}(\mathcal{C}_t)$ for all $t\in[t_0,t_f]$. By adjusting the weight $\beta_t$, we can control the trade-off between editability and safety and $\lambda_t$ can be used to control the smoothness of the control input. Path-safety is formally justified in Proposition~\ref{prop:forward_invariance} below. Such cost functional is considered in~\citep{cohen2020approximate}, where kernel method is used to approximate the value function and optimal control. In contrast, we seek to directly estimate the optimal control from the HJB equation \eqref{eq:hjb} with the cost functional \eqref{eq:cost_functional_aegis}.

Recall that the value function $V$ is defined as \eqref{eq:value_function} with $J$ given by \eqref{eq:cost_functional_aegis}. We are able to derive the more explicit form of the HJB equation.

\begin{proposition}\label{prop:hjb_aegis}
Assume $V$ is smooth and $\lambda_t>0$. Then $V$ satisfies
\begin{equation}
-\partial_t V - \nabla_x V\cdot f + \frac{1}{2\lambda_t}\|\nabla_x V\|_2^2 - \beta_t B(x,t)=0,
\quad V(x,t_f)=\Psi(x),
\label{eq:hjb_forced}
\end{equation}
and the optimal control is given by
\begin{equation}
u^\star(x,t)= -\lambda_t^{-1}\nabla_{x}V(x,t).
\label{eq:feedback_forced}
\end{equation}
\end{proposition}

\begin{proof}
See Appendix~\ref{app:proof_hjb_aegis}.
\end{proof}

\begin{proposition}\label{prop:forward_invariance}
Assume $V$ is smooth, $\mathcal C_t$ varies not too rapidly, $B(x,t)$ blows up fast enough as $x\to\partial \mathcal{C}_t$, and $u^\star$ is given by \eqref{eq:feedback_forced}. Suppose $x(t_0)\in\operatorname{Int}(\mathcal{C}_{t_0})$, then the trajectory of the controlled system \eqref{eq:controlled_dynamics} under the control $u^\star$ satisfies $x(t)\in\operatorname{Int}(\mathcal{C}_t)$ for all $t\in[t_0,t_f]$.
\end{proposition}

\begin{proof}
See Appendix~\ref{app:proof_forward_invariance}.
\end{proof}

\begin{remark}
Proposition~\ref{prop:forward_invariance} establishes the path-safety of the optimal control to the cost functional \eqref{eq:cost_functional_aegis} under certain regularities. It does not automatically guarantee the path-safety of the approximated control derived from the HJB approximation in the next section, and of the output of Algorithm~\ref{alg:aegisflow}. In practice, one may set a safe guard (e.g., an MNC) that intervenes when necessary to ensure path-safety.
\end{remark}

\begin{figure}
  \centering
  \includegraphics[width=1.0\linewidth]{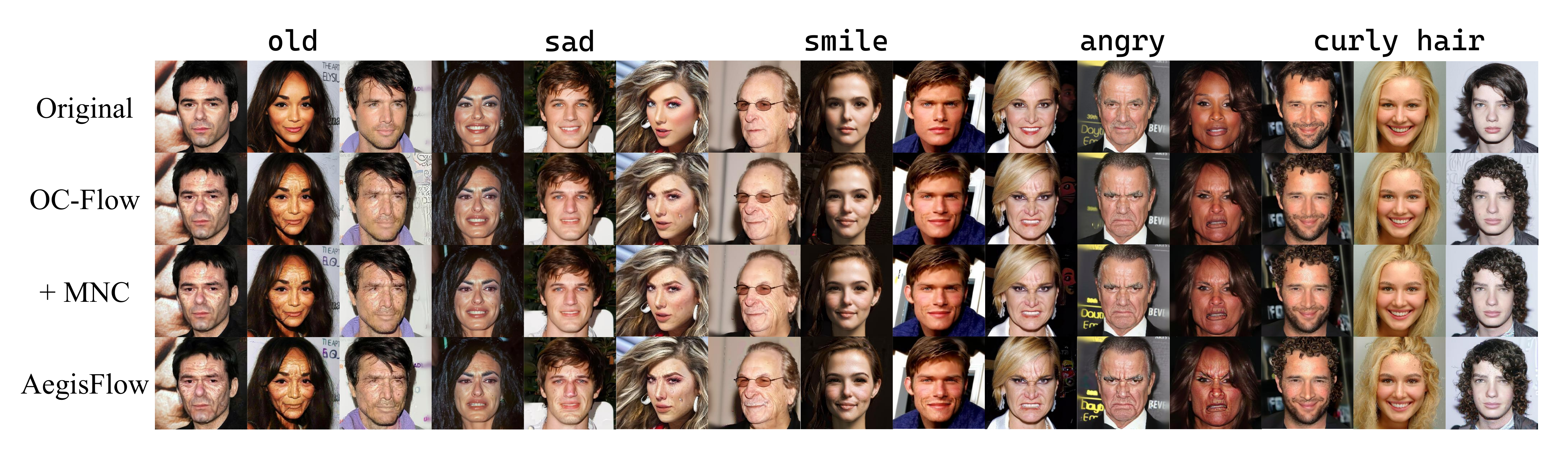}
  \caption{\textbf{Visualization of text-guided image manipulation on CelebA-HQ.} The row "+ MNC" shows the results when MNC is applied on OC-Flow.}\label{fig:celeba_hq}
\vspace{-1.5em}
\end{figure}

\subsection{Tractable HJB Approximation}

In light of Proposition~\ref{prop:hjb_aegis}, it suffices to calculate the gradient of the value function $\nabla_x V$ in order to obtain the optimal control. However, solving the HJB equation \eqref{eq:hjb_forced} is generally intractable. Instead, we seek to directly estimate the optimal control without explicitly solving for $V$. Let us start by characterizing $\nabla_x V$ further.

Following \citet{gao2026terminally}, we introduce the following change of variables:
\begin{equation}\label{eq:comoving_def}
\begin{aligned}
&s(t):=\int_t^{t_f} \lambda_\tau^{-1} \mathrm{d}\tau, \quad \widehat{V}(y,s(t)):=V(\Phi^f_{t_f\to t}(y),t), \quad y:=\Phi^f_{t\to t_f}(x),\\
&G_t(y):=\nabla_x\Phi^f_{t\to t_f}(\Phi^f_{t_f\to t}(y))\nabla_x\Phi^f_{t\to t_f}(\Phi^f_{t_f\to t}(y))^\top, \quad \Gamma_t(y):=\lambda_t\beta_t B(\Phi^f_{t_f\to t}(y),t),
\end{aligned}
\end{equation}
where $\Phi^f_{t_1\to t_2}$ is the flow map induced by the vector field $f$ as defined in the preliminaries. Then we have the following result.

\begin{proposition}
\label{prop:forced_hj_terminal}
Under Eq.~\eqref{eq:comoving_def}, Eq.~\eqref{eq:hjb_forced} transforms to a forced Hamilton-Jacobi equation
\begin{equation}
\partial_s\widehat{V}(y,s) + \frac{1}{2}\nabla_{y}\widehat{V}(y,s)^\top G_t(y)\nabla_{y}\widehat{V}(y,s)=\Gamma_t(y),
\quad \widehat{V}(y,0)=\Psi(y).
\label{eq:forced_hj}
\end{equation}
Moreover,
\begin{equation}
\nabla_{x}V(x,t)=\nabla_x\Phi^{f}_{t\to t_f}(x)^\top\nabla_{y}\widehat{V}(y,s(t)).
\label{eq:grad_pullback}
\end{equation}
\end{proposition}

\begin{proof}
See Appendix~\ref{app:proof_forced_hj_terminal}.
\end{proof}

The expression \eqref{eq:forced_hj} is still difficult to solve exactly in high dimension. Suppose $\Gamma_t(y)\equiv\Gamma$ and $G_t\equiv G$ are constant, then \eqref{eq:forced_hj} admits an explicit solution given by the Hopf-Lax formula~\citep{evans2022partial}:
\[\widehat{V}(y,s) = s\Gamma + \inf_{z}\Big\{\Psi(z) + \frac{1}{2s}\|y-z\|_{G^{-1}}^2\Big\}.\]
Since PDE is locally well-approximated, we have
\[\widehat{V}(y,s) \approx s\Gamma_t(y) + \inf_{z}\Big\{\Psi(z) + \frac{1}{2s}\|y-z\|_{G_t^{-1}}^2\Big\}=:s\Gamma_t(y)+U.\]
If the infimum is attained at a unique point $y^+(y)$ for all $y$, then $U$ is differentiable and its gradient is given by the Moreau-Yosida identity~\citep{rockafellar1998variational}:
\[\nabla_y U = \frac1s G_t^{-1}(y-y^+(y)).\]
Consequently, by \eqref{eq:feedback_forced} and \eqref{eq:grad_pullback}, the optimal control can be approximated by
\begin{equation}
\begin{aligned}
u^\star(x,t) &\approx -\lambda_t^{-1}\nabla_x\Phi^{f}_{t\to t_f}(x)^\top\Big(s\nabla_y\Gamma_t(y) + \frac1s G_t^{-1}(y-y^+(y))\Big)\\
&=-s\beta_t\nabla_x B(x,t) - \lambda_t^{-1}\nabla_x\Phi^{f}_{t\to t_f}(x)^\top \frac1s G_t^{-1}(y-y^+(y)).
\end{aligned}
\label{eq:control_approx}
\end{equation}

If we adopt the first-order approximation $\Psi(z)\approx\Psi(y)+\nabla\Psi(y)\cdot(z-y)$ and assume $G_t\approx I$, then the optimal control can be approximated by
\begin{equation}
u^\star(x,t) \approx -s\beta_t\nabla_x B(x,t) - \lambda_t^{-1}\nabla_x\Psi(y).
\label{eq:control_approx_gd}
\end{equation}

\citet{gao2026terminally} deal with a similar problem without considering path-safety. The introduction of the barrier function term turns an unforced HJB equation into a forced one, which enhance the difficulty to obtain the approximation. We summarize the procedure of AegisFlow in Algorithm~\ref{alg:aegisflow}. 

\begin{remark} 
In Algorithm~\ref{alg:aegisflow}, we implement a heuristic substitution that replaces the nominal flow $\Phi^f$ in \eqref{eq:control_approx_gd} with the iteratively updated controlled flow $\Phi^{f+u^k}$. This modification ensures that the terminal loss landscape is evaluated exactly along the controlled trajectory. As $k$ increases and $u^k$ converges toward the optimal control, exploiting $\Phi^{f+u^k}$ substantially reduces the compounding linearization error that would otherwise arise from tracking solely along the unperturbed flow $\Phi^f$.
\end{remark}

\section{Experiments}

\subsection{Text-Guided Image Manipulation with Identity Preservation}

\paragraph{Dataset and Evaluation.} We conduct experiments on the celebrity face dataset CelebA-HQ~\citep{karras2017progressive} with text guidance \texttt{\{old, smile, sad, angry, curly hair\}}. Following the same setup as in~\citep{liu2023flowgrad}, the alignment between the generated image and text prompt is evaluated by the CLIP model~\citep{radford2021learning}. As in~\citep{kim2022diffusionclip}, we measure the perceptual similarity between the source and generated image by LPIPS~\citep{zhang2018unreasonable} and identity similarity by ID, which is computed using ArcFace~\citep{deng2019arcface}. Implementation details are deferred to Appendix~\ref{app:implementation_image}.

\begin{table}
  \caption{\textbf{Comparison of AegisFlow with state-of-the-art methods on CelebA-HQ.} The best results are in \textbf{bold}. We also include the results of directly applying MNC for reference.}
  \label{main-results}
  \centering
  \begin{tabular}{lllll}
    \toprule
    Method     & CLIP $\uparrow$ & LPIPS $\downarrow$ & ID $\uparrow$  \\
    \midrule
    CG + RF~\citep{dhariwal2021diffusion} & 0.292 & 0.346 & 0.643     \\
    CG + LDM~\citep{dhariwal2021diffusion} & 0.298 & 0.383 & 0.513      \\
    DiffusionCLIP~\citep{kim2022diffusionclip} & 0.285 & 0.398 & 0.659  \\
    StyleCLIP + e4e~\citep{patashnik2021styleclip} & 0.267 & 0.359 & 0.704 \\
    FlowGrad + RF~\citep{liu2023flowgrad} & 0.299 & 0.302 & 0.737 \\
    OC-Flow~\citep{wang2024training} & 0.302 & 0.207 & 0.732 \\
    \midrule
    OC-Flow~\citep{wang2024training} + MNC & 0.304 & 0.193 & \textbf{0.920} \\
    \textbf{AegisFlow (Ours)} & \textbf{0.309} & \textbf{0.188} & 0.901 \\
    \bottomrule
  \end{tabular}
\end{table}

\paragraph{Main Results.} We benchmark AegisFlow against state-of-the-art methods on CelebA-HQ (Figure~\ref{fig:celeba_hq}). The results are summarized in Table~\ref{main-results} and Figure~\ref{fig:h_value}. Our method achieves the best CLIP and LPIPS score, indicating that it improves text alignment while preserving subject identity better than prior methods. 
We also compare against the myopic safety baseline, i.e. applying the plug-and-play MNC to OC-Flow. While MNC attains the highest ID, it underperforms AegisFlow on both CLIP and LPIPS. 

\begin{figure}[t]
  \centering
  \includegraphics[width=0.9\linewidth]{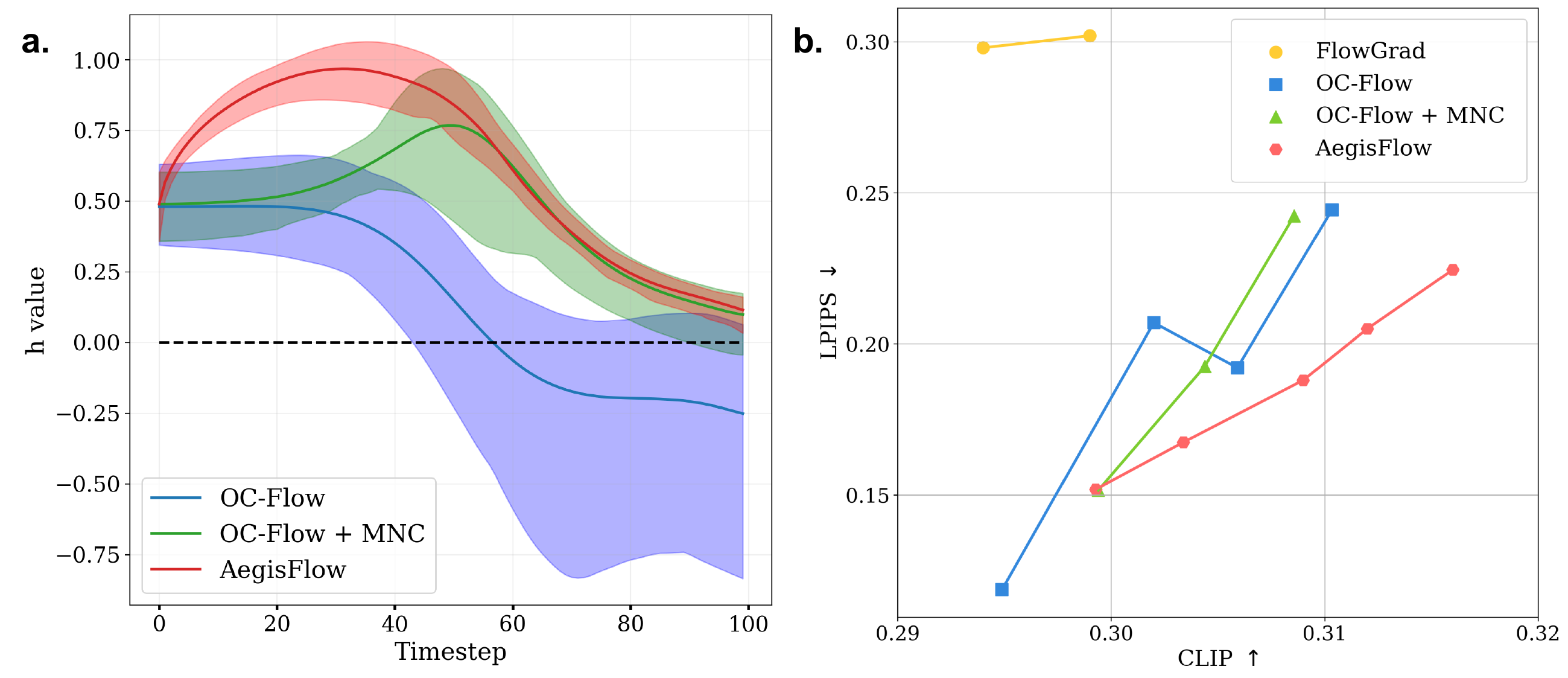}
  \caption{\textbf{Quantitative Analysis on CelebA-HQ.} \textbf{a.} The value of the function $h(x,t)$ at each timestep. The solid lines represent the mean value for the three methods and the shallow bands represent the range of variation. The x-axis is emphasized by a dash line since the forward invariance of $\mathcal{C}_t$ requires $h(x,t) \geq 0$ for all $t$. \textbf{b.} AegisFlow achieves a more favorable balance between text alignment (CLIP) and structure preservation (LPIPS). Connected markers represents different hyperparameter settings (See Appendix~\ref{app:ablation_celeba}).}\label{fig:h_value}
\end{figure}

To better understand AegisFlow, we record the value of the function $h(x,t)$ at each timestep, as shown in Figure~\ref{fig:h_value}. It can be seen that AegisFlow acts the most proactively to maintain a large margin from the boundary of the safe set. Although MNC and AegisFlow both renders safety for the most part of the trajectory, AegisFlow achieves better performance by maintaining a larger safety margin and thus allowing more aggressive optimization towards the target prompt. Comparisons on computational efficiency are deferred to Appendix~\ref{app:computation_efficiency}.

\subsection{Maze Navigation with Safety Constraints}

\begin{figure}[t]
  \centering
  \includegraphics[width=0.9\linewidth]{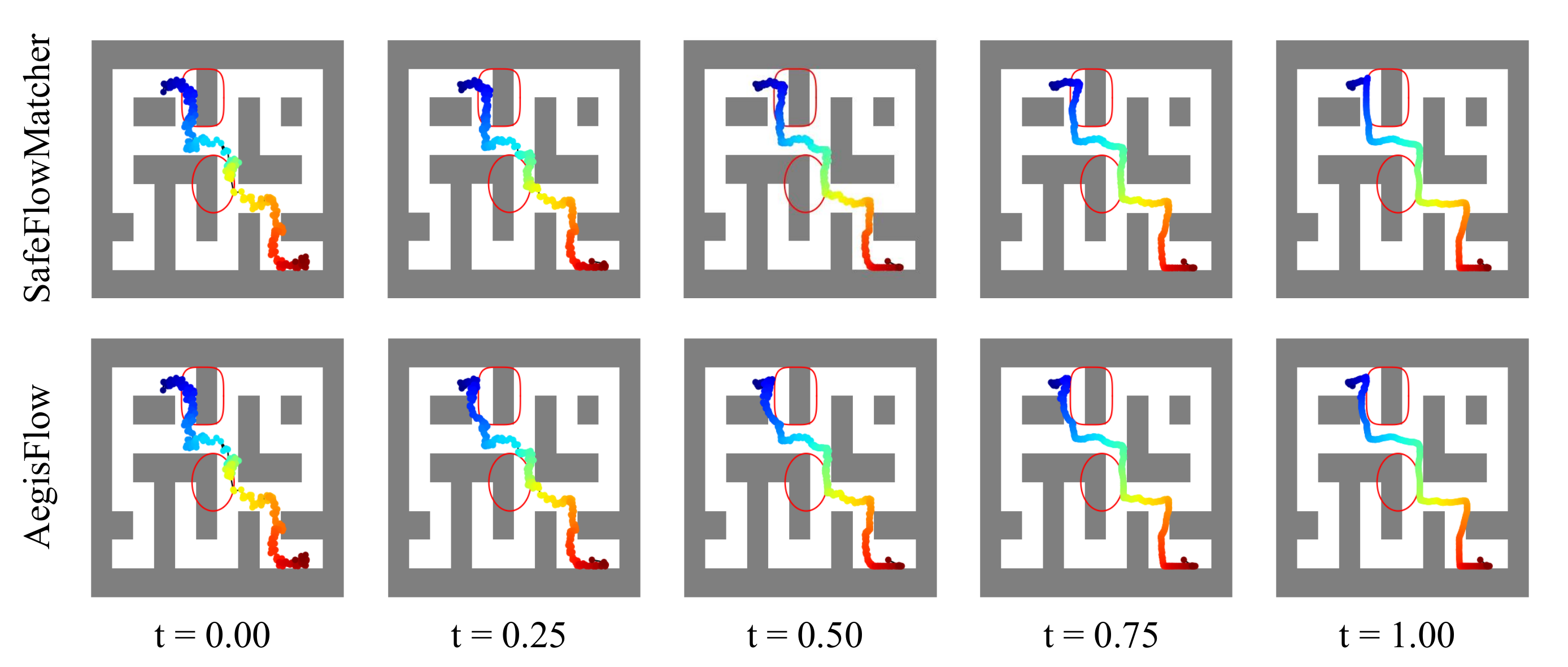}
  \caption{\textbf{Visualization of path generation process in Maze2D.} Red circles represent the obstacles. AegisFlow behaves more proactively to maintain a large margin from the boundary of the safe set in the early stages of generation.}\label{fig:maze}
\end{figure}

\paragraph{Dataset and Evaluation.} We evaluate AegisFlow on a 2D maze navigation task with safety constraints. The experiment is conducted on \texttt{maze\_large\_v1} of the Maze2D environment from D4RL~\citep{fu2020d4rl}. Following \citet{xiao2023safediffuser,yang2025safeflowmatcher}, we introduce two obstacle-based safety constraints given by the following functions:
\[\textbf{BS1}:(\frac{x-x_0}{a})^2 + (\frac{y-y_0}{b})^2 \ge 1,\quad\textbf{BS2}:(\frac{x-x_0}{a})^4 + (\frac{y-y_0}{b})^4 \ge 1.\]
The minimum value of the above functions along the planned trajectory is recorded to evaluate the safety performance. We also report the D4RL normalized score, the average curvature $\kappa$ and the average acceleration $a$ of the trajectory. Implementation details are deferred to Appendix~\ref{app:implementation_maze}.

\begin{table}[t]
  \caption{\textbf{Comparison of AegisFlow with state-of-the-art methods on Maze2D.} The best results are in \textbf{bold}. BS1 and BS2 being positive ($\ge0$) indicates the trajectory is safe with respect to the corresponding constraint. $a\pm b$ means $a$ is the mean and $b$ is the standard deviation over 100 trials.}
  \label{maze-results}
  \centering
  \begin{tabular}{llllll}
    \toprule
    Method     & BS1 $\uparrow$ & BS2 $\uparrow$ & Score $\uparrow$ & $\kappa$ $\downarrow$ & $a$ $\downarrow$ \\
    \midrule
    Diffuser~\citep{janner2022planning} & -0.825 & -0.784 & 1.572$\pm$0.288 & 77.04$\pm$4.30 & 86.68$\pm$3.81 \\
    FlowMatcher & -0.324 & -0.904 & 1.632$\pm$0.003 & 73.51$\pm$1.02 & 88.45$\pm$0.60 \\
    \midrule
    ROS-SafeDiffuser~\citep{xiao2023safediffuser} & 0.010 & 0.010 & 1.435$\pm$0.502 & 75.15$\pm$6.67 & 422.87$\pm$86.70 \\
    RES-SafeDiffuser~\citep{xiao2023safediffuser} & 0.010 & 0.010 & 1.442$\pm$0.451 & 80.30$\pm$13.06 & 398.17$\pm$1060.86 \\
    TVS-SafeDiffuser~\citep{xiao2023safediffuser} & -0.003 & -0.003 & 1.506$\pm$0.405 & 78.72$\pm$7.80 & 124.51$\pm$34.22 \\
    SafeFlowMatcher~\citep{yang2025safeflowmatcher} & 0.020 & 0.041 & 1.632$\pm$0.003 & 51.97$\pm$2.55 & 65.71$\pm$2.45 \\
    \midrule
    \textbf{AegisFlow (Ours)} & \textbf{0.049} & \textbf{0.071} & \textbf{1.635$\pm$0.003} & \textbf{44.72$\pm$2.10} & \textbf{50.29$\pm$6.17} \\
    \bottomrule
  \end{tabular}
\end{table}

\paragraph{Main Results.} The results are summarized in Table~\ref{maze-results}. AegisFlow achieves the strongest safety margins (highest BS1/BS2) and the best generation quality among all methods, indicating safer and smoother plans. Compared with SafeFlowMatcher (which adopts MNC), AegisFlow improves both safety metrics and reduces trajectory curvature and acceleration. Figure~\ref{fig:maze} shows that AegisFlow maintains a larger safety buffer early in the rollout, which aligns with its non-myopic formulation. Comparisons on computational efficiency are deferred to Appendix~\ref{app:computation_efficiency}.

\subsection{Robotic Manipulation with Safety Constraints}

\paragraph{Dataset and Evaluation.} We implement AegisFlow on a robotic manipulation task with novel test-time constraints. The experiment is conducted on the \texttt{Avoiding} environment from D3IL~\citep{jia2024towards}. The task is to control a robot manipulator to reach the target region while avoiding collision with obstacles. We also introduce the same novel test-time constraints as in \citet{romer2024diffusion}. We report the success rate of reaching the target (goal), of reaching the target while satisfying the safety constraints (safe), the number of timesteps the manipulator takes to reach the target (timesteps), and the number of timesteps the manipulator violates the safety constraints (violation). Implementation details are deferred to Appendix~\ref{app:implementation_robot}.

\begin{figure}
  \centering
  \includegraphics[width=0.75\linewidth]{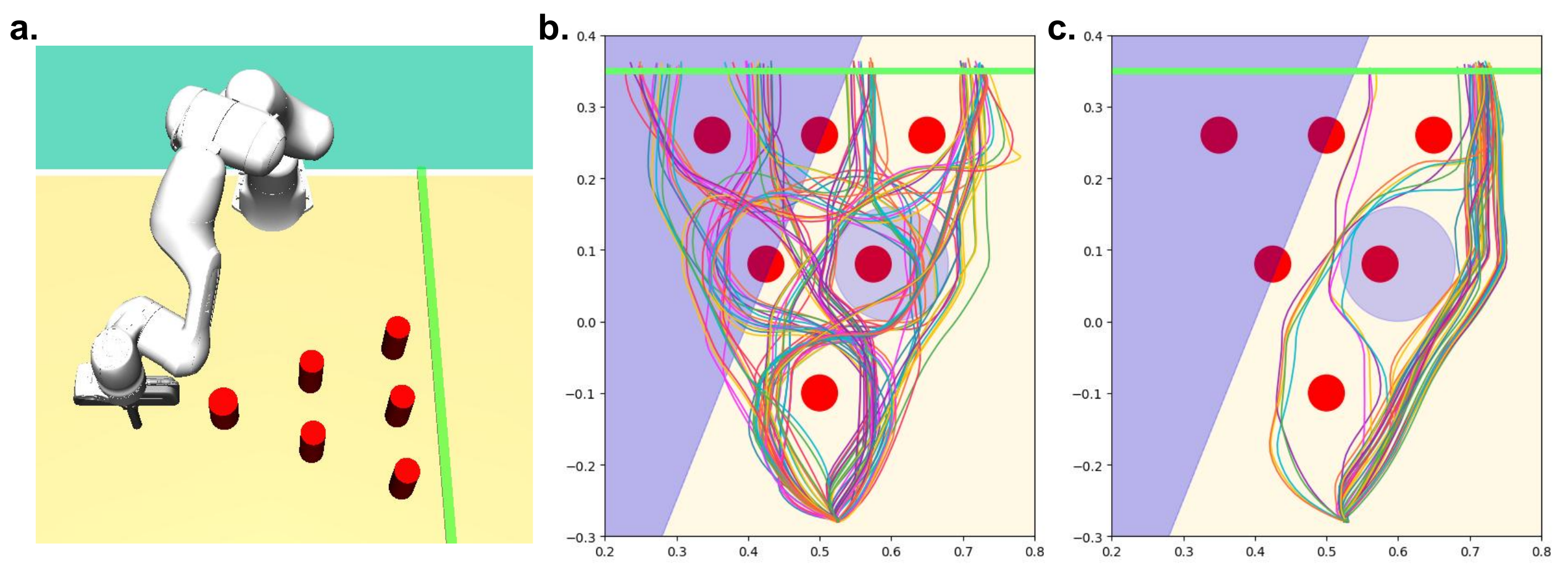}
  \caption{\textbf{Qualitative Results for Robotic Manipulation.} \textbf{a.} Visualization of the robotic manipulation task. The object is to reach the green line while avoiding collision with the red obstacles. \textbf{b.} Generated trajectories by Flowmatcher. \textbf{c.} Generated trajectories by AegisFlow. The novel test-time constraints are that the robot should not enter the blue regions.}\label{fig:avoiding}
\end{figure}

\begin{table}[!]
  \caption{\textbf{Comparison of AegisFlow with state-of-the-art methods on Avoiding.} The best results are in \textbf{bold}. $a\pm b$ means $a$ is the mean and $b$ is the standard deviation over 100 trials.}
  \label{avoiding-results}
  \centering
  \begin{tabular}{lllll}
    \toprule
    Method     & Timesteps $\downarrow$ & Goal $\uparrow$ & Safe $\uparrow$ & Violation $\downarrow$ \\
    \midrule
    Diffuser~\citep{janner2022planning} & 76.7$\pm$12.7 & \textbf{1.00} & 0.07 & 17.8$\pm$12.1 \\
    Flowmatcher & 64.2$\pm$9.4 & \textbf{1.00} & 0.22 & 14.8$\pm$11.5 \\
    \midrule
    DOM~\citep{giannone2023aligning} & 79.1$\pm$13.9 & \textbf{1.00} & 0.96 & 0.1$\pm$0.5 \\
    Model-free~\citep{romer2024safe} & 76.1$\pm$12.1 & 0.99 & 0.07& 18.0$\pm$12.0 \\
    DPCC-C~\citep{romer2024diffusion} & 69.0$\pm$12.9 & \textbf{1.00} & \textbf{0.98} & \textbf{0.0$\pm$0.3} \\
    \midrule
    \textbf{AegisFlow (Ours)} & \textbf{62.6$\pm$2.0} & \textbf{1.00} & \textbf{0.98} & 0.1$\pm$0.5\\
    \bottomrule
  \end{tabular}
\end{table}

\paragraph{Main Results.} Table~\ref{avoiding-results} shows that AegisFlow maintains the best goal success and safety rate while achieving the fewest timesteps. Compared with DPCC-C, AegisFlow maintains the same safety level with faster task completion. Relative to unsafe baselines, AegisFlow substantially reduces violations while preserving goal success. Figure~\ref{fig:avoiding} visualizes that AegisFlow respects the novel constraints by steering trajectories away from forbidden regions.

\section{Conclusion}

In this paper, we presented AegisFlow, a training-free guidance framework that addresses the limitations of myopic controllers in path-safe guided flow matching. By casting guided generation as a safety-critical optimal control problem, AegisFlow integrates constraints into the cost functional. A tractable HJB approximation is derived to obtain the optimal control. Comprehensive empirical evaluations across diverse domains demonstrate the effectiveness of AegisFlow. Future works may improve HJB approximation, enhance computational efficiency, and extend this control-theoretic framework to broader settings and more complex modalities.

\section{Acknowledgements}

This work was supported by National Natural Science Foundation of China (62325305), and the BNRist project (No. BNR2024TD03003).

\bibliographystyle{plainnat}
\bibliography{nips_2026}

\newpage
\appendix

\section{Background on Control Barrier Functions}\label{app:cbf}

\subsection{Control Barrier Functions} 

Control Barrier Functions (CBFs)~\citep{ames2016control,ames2019control} are a powerful tool to enforce hard safety constraints in control theory. Specifically, given a control system
\begin{equation}\label{eq:control_system_cbf}
\dot{x} = f(x) + g(x)u,
\end{equation}
we want to ensure that the state $x(t)$ remains within a safe set
\begin{equation}\label{eq:safe_set}
\mathcal{C}:=\{x\in\mathbb{R}^n: h(x)\geq0\}
\end{equation}
for all $t\in[t_0, t_f]$. If so, we say that the set $\mathcal{C}$ is \emph{forward invariant}. We also allow $\mathcal{C}$ to be time-varying, i.e. $\mathcal{C}_t:=\{x\in\mathbb{R}^n: h(x,t)\geq0\}$, and in this case we say $\mathcal{C}_t$ is \emph{forward invariant} if $x(t)\in \mathcal{C}_t$ for all $t\in[t_0, t_f]$. The intuition of CBFs is to ensure that when evaluated along the trajectories of the system, the gradient of $h$ is nonnegative whenever $h(x)=0$, which prevents the system from leaving the safe set. We shall introduce the formal definition of CBFs as follows.

\begin{definition}[Extended Class $\mathcal{K}_\infty$ Function]
A function $\alpha:\mathbb{R}\to\mathbb{R}$ is an extended class $\mathcal{K}_\infty$ function if it is strictly increasing and with $\alpha(0)=0$.
\end{definition}

\begin{definition}[Control Barrier Function]
Let $\mathcal{C}\subset \mathcal{D}\subset\mathbb{R}^n$ be with the form of \eqref{eq:safe_set} and $h:\mathcal{D}\to\mathbb{R}$ be a continuously differentiable function, then $h$ is a control barrier function (CBF) if there exists an extended class $\mathcal{K}_\infty$ function $\alpha$ such that for the control system \eqref{eq:control_system_cbf}, we have
\begin{equation}\label{eq:cbf_condition}
\sup_{u\in U}[\nabla h(x)\cdot(f(x)+g(x)u)+\alpha(h(x))]\geq0
\end{equation}
for all $x\in \mathcal{D}$. For a CBF $h$, define the set of safe controls at $x$ as
\[K_{\text{CBF}}(x) := \{u\in U : \nabla h(x)\cdot(f(x)+g(x)u)+\alpha(h(x))\geq0\}.\]
\end{definition}

\begin{theorem}[\citet{ames2016control}]\label{thm:cbf_invariance}
Let $\mathcal{C}\subset \mathcal{D}\subset\mathbb{R}^n$ be with the form of \eqref{eq:safe_set} and $h:\mathcal{D}\to\mathbb{R}$ be a continuously differentiable function. If $h$ is a CBF and $\nabla h(x)\neq0, \forall x\in\partial \mathcal{C}$, then any Lipschitz continuous controller $u(x)\in K_{\text{CBF}}(x)$ renders the set $\mathcal{C}$ forward invariant. Moreover, $\mathcal{C}$ is asymptotically stable, i.e. if $h(x(t))<0$ for some $t$, then $h(x(t))\to0$ as $t\to\infty$.
\end{theorem}
The result above can be easily generalized to time-varying safe sets by introducing an auxiliary state to represent time.

\subsection{Min-norm Controller}

Min-norm controller~\citep{freeman2008robust,ames2014rapidly} is a widely used method to render forward invariance in control theory. Given a dynamical system and a nominal control $u_{\text{nom}}$, by imposing an additive safe control $u_{\text{safe}}$, we seek to ensure that $\mathcal{C}_t$ under
\[\dot{x} = f(x,t)+u_{\text{nom}}(t)+u_{\text{safe}}(t)\]
is forward invariant.

It is desired that $u_{\text{safe}}$ can be as small as possible in order to preserve the editability of the original nominal control, which naturally leads to the following optimization problem at each timestep:
\begin{equation}\label{eq:cbf_qp_slack}
\begin{aligned}\min_{u_{\text{safe}}(t),\xi} &\quad \frac12\|u_{\text{safe}}\|_2^2+\rho_t\xi\\
\text{s.t.} &\quad a_t\cdot u_{\text{safe}}(t)+b_t+\xi\geq0, \quad \xi\geq0.
\end{aligned}
\end{equation}
where we introduce the slack variable $\xi$ to relax the CBF constraint for the sake of stability and robustness, and we denote for brevity
\[a_t:=\nabla_x h(x,t), \quad b_t:=\nabla_x h(x,t)\cdot(f(x,t)+u_{\text{nom}}(t))+\partial_t h(x,t)+\alpha(h(x,t)).\]

Min-norm controller imposes the safe control $u_{\text{safe}}(t)$ by solving the optimization problem \eqref{eq:cbf_qp_slack}.

\section{Proof of the Theoretical Results}

\subsection{Proof of Proposition~\ref{prop:hjb_aegis}\label{app:proof_hjb_aegis}}

\begin{proof}
Substitute the cost functional \eqref{eq:cost_functional_aegis} and the dynamics \eqref{eq:controlled_dynamics} into the HJB equation \eqref{eq:hjb}, we have
\[-\partial_t V = \min_{u \in U} \left\{\frac{\lambda_t}2\|u\|_2^2 + \beta_t B + \nabla V \cdot (f+u) \right\}, \quad V(t_f,x)=\Psi(x).\]

The term inside the minimization is a quadratic function of $u$, which allows us to derive the explicit form of the value function and the optimal control.
\[-\partial_t V = \nabla_x V\cdot f+\beta_t B-\frac{1}{2\lambda_t}\|\nabla_x V\|_2^2,\]
\[u^\star(x,t)=-\lambda_t^{-1} \nabla_x V(x,t),\]
which completes the proof.
\end{proof}

\subsection{Proof of Proposition~\ref{prop:forward_invariance}}\label{app:proof_forward_invariance}

\begin{proof}
Suppose $\mathcal{C}_t$ is not forward invariant. Let $\bar{t}:=\inf\{t\in[t_0,t_f]:x(t)\notin \mathcal{C}_t\}$ be the first time the trajectory leaves the safe set. Since $x(t_0)\in\operatorname{Int}(C_{t_0})$, we have $\bar{t}>t_0$ and $h(x(\bar{t}),\bar{t})=0$. Taking derivative of $V$ along the trajectory gives
\[\frac{\mathrm{d}}{\mathrm{d}t}V(x(t),t)=\nabla_x V\cdot (f+u^\star)+\partial_t V\le-\beta_t B(x(t),t),\]
where the inequality follows from \eqref{eq:hjb_forced} and \eqref{eq:feedback_forced}. Since $\beta_t\ge0$ and $B(x(t),t)\ge0$ for all $t\in[t_0,\bar{t})$, we have $\frac{\mathrm{d}}{\mathrm{d}t}V(x(t),t)\le0$. 
On the other hand, since $\mathcal C_t$ varies not too rapidly and $B(x,t)$ blows up fast enough as $x\to\partial \mathcal{C}_t$, we have $V(x(t),t)\to\infty$ as $t\to\bar{t}$. This contradicts the fact that $V$ is non-increasing along the trajectory. 
\end{proof}
\begin{remark}
For example, suppose $B(x,t)\ge\frac{c}{h(x,t)^2}$ for some $c>0$ and $d(t):=\mathrm{dist}(x(t),\partial \mathcal{C}_t)$ is $L$-Lipschitz. Suppose $\int_{t_0}^{\bar{t}}B(x(t),t)\mathrm{d}t<\infty$, then it follows that $d^{-1}\in L^2([t_0,\bar{t}])$. On the other hand, $\dot x=f+u^\star\in L^2$. Thus, $x$ is absolutely continuous and hence $d$ is absolutely continuous. Moreover, $|\dot d|\le\|\dot x\|+L\in L^2$. By the Cauchy-Schwarz inequality, it follows that $|\dot d|/d\in L^1([t_0,\bar{t}])$, which implies that $\log d$ is absolutely continuous. However, since $d(t)\to0$ as $t\to\bar{t}$, it follows that $\log d(t)\to-\infty$ as $t\to\bar{t}$, absurd.
\end{remark}

\subsection{Proof of Proposition~\ref{prop:forced_hj_terminal}}\label{app:proof_forced_hj_terminal}

\begin{proof}
Eq.~\eqref{eq:grad_pullback} can be derived by the chain rule and the definition of the flow map. Then it follows that
\[\|\nabla_x V(x,t)\|^2_2 = \nabla_{y}\widehat{V}(y,s)^\top G_t(y)\nabla_{y}\widehat{V}(y,s).\]
Notice that $\Phi^f_{t\to t_f}(x(t))=x(t_f)=\text{const.}$, taking the time derivative on both sides gives
\[\nabla_x\Phi^f_{t\to t_f}(x(t))\cdot f + \partial_t \Phi^f_{t\to t_f}(x(t))=0.\]

Taking the time derivative on both sides of $\widehat{V}(\Phi^f_{t\to t_f}(x),s(t))=V(x,t)$, we have
\begin{align*}
\partial_t V(x,t) &= \partial_s \widehat{V}(y,s(t))\cdot s'(t) + \nabla_y \widehat{V}(y,s(t))\cdot \partial_t \Phi^f_{t\to t_f}(x)\\
&= -\lambda_t^{-1}\partial_s \widehat{V}(y,s(t)) - \nabla_y \widehat{V}(y,s(t))\cdot \nabla_x\Phi^f_{t\to t_f}(x)\cdot f.
\end{align*}
Substituting the above identities into \eqref{eq:hjb_forced} gives \eqref{eq:forced_hj}, which completes the proof.
\end{proof}

\section{Implementation Details}\label{app:implementation}

\subsection{Text-Guided Image Manipulation with Identity Preservation}\label{app:implementation_image}

Following~\citep{liu2023flowgrad, wang2024training}, we adopt the pretrained Rectified Flow (RF)~\citep{liu2022flow} as the generative prior. To ensure identity preservation, we define the safe set as $\mathcal{C}_t:=\{x:h(x,t)\ge0\}$, with
\[h(x,t):=\epsilon_t-[\operatorname{dist}(x,x_{\text{src}})]^2,\]
where $\operatorname{dist}(\cdot,\cdot)$ is the ArcFace distance, $x_{\text{src}}$ is the source image, and $\epsilon_t$ is a time-varying margin. Based on the dataset statistics, we set
\[\epsilon_t:=\epsilon_1+\frac{\epsilon_0-\epsilon_1}{1+\exp(\kappa_\epsilon(t-\tau_0))},\]
with $\epsilon_0=2.5, \epsilon_1=0.2, \kappa_\epsilon=10, \tau_0=0.55$.
Note that $\epsilon_0$ is set to be larger than the maximum ArcFace distance between any image and its inversion to ensure the condition of Proposition~\ref{prop:forward_invariance} is satisfied, i.e. $x_0 \in \operatorname{Int}(\mathcal{C}_0)$. The barrier function is set to be
\[B(x,t) := \frac{\sigma\left(-\kappa_B h(x,t)\right)}{h(x,t)},
\]
where $\sigma(\cdot)$ is the sigmoid function, $\kappa_B=5$, and $\delta=0.15$. With such a barrier function, Proposition~\ref{prop:forward_invariance} guarantees the ArcFace similarity between the generated image and the source image would be larger than $1-\frac{\epsilon_{t_f}+\delta}{2}\approx0.825$, which is aligned with the empirical results in Table~\ref{main-results}.

The terminal loss is defined as $\Psi(x):=-0.7\cdot\operatorname{CLIP}(x, c)$, where $c$ is the target text prompt. The schedules are set to be $\lambda_t\equiv0.6$ and $\beta_t\equiv0.1$. Following \citet{wang2024training}, we set $M=15$, $\eta=2.5$ and $\nu=0.995$.

For the min-norm controller, we adopt the same safe set with $\rho_t=1$ for all $t$. The $\mathcal{K}_\infty$ function $\alpha$ is set to be $\alpha(r)=\gamma \cdot r$ with $\gamma=0.5$. The experiments are conducted on 4 NVIDIA GeForce RTX 4090 GPUs with 24GB memory.

\subsection{Maze Navigation with Safety Constraints}\label{app:implementation_maze}

We train a flow matching model to be the generative prior. The training strategy is summarized in Table~\ref{tab:train_maze}. The barrier function is set to be $B(\tau,t):=\sum_{x\in\tau}\big[(h_1(x)-t+\delta)^{-1}+2(h_2(x)-t+\delta)^{-1}\big]$, where $h_1$ and $h_2$ are the two functions defined in \textbf{BS1} and \textbf{BS2}, respectively, and $\delta=0.05$. Note that the term corresponding to \textbf{BS2} is weighted more heavily because it is harder to satisfy~\citep{xiao2023safediffuser}. Since \texttt{score} is computed via interaction with the environment and thus is non-differentiable, we set the terminal loss to be a linear combination of the curvature and acceleration, i.e. $\Psi(\tau):=0.01(\lambda_\kappa \kappa(\tau)+\lambda_a a(\tau))$, where $\tau$ is the planned trajectory, $\lambda_\kappa=0.6$ and $\lambda_a=0.4$. The schedule is set to be $\lambda_t\equiv2$, $\beta_t\equiv0.01$, $M=15$, $\eta=0.05$ and $\nu=0.995$. The experiments are conducted on 4 NVIDIA GeForce RTX 4090 GPUs with 24GB memory.

\begin{table}[h]
    \caption{\textbf{Training strategy for the flow matching model in maze navigation.}}
    \label{tab:train_maze}
    \centering
    \begin{tabular}{ll}
    \toprule
        U-Net channels & $(1, 4, 8)\times32$ \\
        Loss type & \texttt{L2}\\
        Training steps & 1M \\
        Steps per epoch & 5K \\
        Batch size & 64 \\
        Learning rate & 2e-4 \\
        EMA decay & 0.995 \\
        Diffusion steps & 256 \\
        Planning horizon & 384 \\
    \bottomrule
    \end{tabular}
\end{table}

\subsection{Robotic Manipulation with Safety Constraints}\label{app:implementation_robot}

We train a flow matching model to be the generative prior. The barrier function is set to be $B(\tau,t):=\sum_{x\in\tau}\big[(h_1(x)-t+\delta)^{-1}+(h_2(x)-t+\delta)^{-1}\big]$, where $h_1$ and $h_2$ are the two functions corresponding to the halfspace and circular constraints, respectively, and $\delta=0.05$. Since \texttt{timesteps} is computed via interation with the environment and thus is non-differentiable, we set the terminal loss to be the curvature of the path, i.e. $\Psi(\tau):=0.1\kappa(\tau)$, where $\tau$ is the planned trajectory. The schedule is set to be $\lambda_t\equiv1.5$, $\beta_t\equiv0.5$, $M=7$, $\eta=0.2$ and $\nu=0.995$. The experiments are conducted on 4 NVIDIA GeForce RTX 4090 GPUs with 24GB memory.

\section{Additional Experimental Results}

\subsection{Computational Efficiency Comparisons}\label{app:computation_efficiency}

We compare the computational efficiency of AegisFlow with the baselines. The results are shown in Table~\ref{tab:computation_efficiency}. For the image manipulation task, AegisFlow takes slightly more time than OC-Flow, and is more efficient than OC-Flow + MNC. For the maze navigation task, AegisFlow is slower than SafeDiffuser and SafeFlowmatcher, which is expected since AegisFlow attempts to estimate the solution to the forced HJB equation, while SafeDiffuser and SafeFlowmatcher only solve a simple quadratic program at each timestep in one go. Indeed, they mainly consider constraint satisfaction in a myopic manner, while AegisFlow seeks to find a globally optimal solution that balances constraint satisfaction and task performance. 

The experiment is conducted on 2 NVIDIA RTX 6000D GPUs with 84GB memory.

\begin{table}
  \caption{\textbf{Computational efficiency comparisons.} $a\pm b$ means $a$ is the mean and $b$ is the standard deviation over 1000 trials (CelebA-HQ) or 100 trials (Maze2D).}
  \label{tab:computation_efficiency}
  \centering
  \begin{tabular}{lll}
    \toprule
    Dataset & Method & Computation Time (s) \\
    \midrule
    \multirow{3}{*}{CelebA-HQ}
    & OC-Flow & 421.02$\pm$25.56 \\
    & OC-Flow + MNC & 518.26$\pm$25.54 \\
    & AegisFlow & 433.81$\pm$25.57 \\
    \midrule
    \multirow{3}{*}{Maze2D}
    & SafeDiffuser & 4.78$\pm$0.42 \\
    & SafeFlowmatcher & 4.21$\pm$0.39 \\
    & AegisFlow & 44.15$\pm$1.96 \\
    \bottomrule
  \end{tabular}
\end{table}

\subsection{Metrics Comparisons for Text-Guided Image Manipulation}\label{app:ablation_celeba}

The detailed metrics comparisons for text-guided image manipulation on CelebA-HQ are shown in Table~\ref{ablation_celeba}.

\begin{table}
  \caption{\textbf{Detailed metrics comparisons for text-guided image manipulation on CelebA-HQ.}}
  \label{ablation_celeba}
  \centering
  \begin{tabular}{lllll}
    \toprule
    Method & Hyperparameters & CLIP $\uparrow$ & LPIPS $\downarrow$ & ID $\uparrow$  \\
    \midrule
    \multirow{2}{*}{FlowGrad} & + RF & 0.299 & 0.302 & 0.737 \\ & + LDM & 0.294 & 0.298 & 0.743 \\
    \midrule
    \multirow{4}{*}{OC-Flow} & $\alpha=0.6$ & 0.295 & 0.119 & 0.822 \\ & $\alpha=0.7$ & 0.302 & 0.207 & 0.732 \\ & $\alpha=0.8$ & 0.306 & 0.192 & 0.722 \\ & $\alpha=0.9$ & 0.310 & 0.244 & 0.662 \\
    \midrule
    \multirow{3}{*}{OC-Flow + MNC} & $\alpha=0.7$ & 0.299 & 0.152 & 0.937 \\ & $\alpha=0.8$ &0.304 & 0.193 & 0.920 \\ & $\alpha=0.9$ & 0.309 & 0.242 & 0.901 \\
    \midrule
    \multirow{5}{*}{AegisFlow} & $\lambda=1.0$ & 0.299 & 0.152 & 0.903 \\ & $\lambda=0.8$ & 0.303 & 0.167 & 0.901 \\ & $\lambda=0.6$ & 0.309 & 0.188 & 0.901 \\ & $\lambda=0.5$ & 0.312 & 0.205 & 0.902 \\ & $\lambda=0.4$ & 0.316 & 0.224 & 0.903 \\
    \bottomrule
  \end{tabular}
\end{table}

\end{document}